\documentclass[10pt,journal,compsoc]{IEEEtran}
\IEEEoverridecommandlockouts
\usepackage{cite}
\usepackage{amsmath,amssymb,amsthm,amsfonts}
\usepackage[linesnumbered,ruled]{algorithm2e} 
\usepackage{graphicx}
\usepackage{textcomp}
\usepackage{xcolor}
\def\BibTeX{{\rm B\kern-.05em{\sc i\kern-.025em b}\kern-.08em
    T\kern-.1667em\lower.7ex\hbox{E}\kern-.125emX}}
\usepackage{booktabs} 
\usepackage{hyperref}  

\newtheorem{theorem}{Theorem}
\newtheorem{theorem*}{Theorem}
\newtheorem{lemma}{Lemma}
\newtheorem{proposition}{Proposition}
\newtheorem{corollary}{Corollary}

\newtheorem{example}{Example}

\begin{document}

\title{Theoretical Investigation of Composite Neural Network}

\author{
\IEEEauthorblockN{
Ming-Chuan Yang\IEEEauthorrefmark{1}, 
and Meng Chang Chen\IEEEauthorrefmark{2}  }\\
\IEEEauthorblockA{ \textit{Institute of Information Science, Academia Sinica}, Taiwan}\\
\IEEEauthorblockA{\IEEEauthorrefmark{1}{mingchuan@iis.sinica.edu.tw},
\IEEEauthorrefmark{2}{mcc@iis.sinica.edu.tw}}
}

\maketitle

\begin{abstract}
This work theoretically investigates the performance of a composite neural network.
A composite neural network is a rooted directed acyclic graph combining a set of pre-trained and non-instantiated neural network models, where a pre-trained neural network model is well-crafted for a specific task and targeted to approximate a specific function with instantiated weights.  
 The advantages of adopting such a pre-trained model in a composite neural network are two folds. One is to benefit from other's intelligence and diligence, and the other is saving the efforts in data preparation and resources and time in training. However, the overall performance of composite neural network is still not clear. In this work, we prove that a composite neural network, with high probability, performs better than any of its pre-trained components under certain assumptions. In addition, if an extra pre-trained component is added to a composite network, with high probability the overall performance will be improved. In the empirical evaluations, distinctively different applications support the above findings.   
\end{abstract}


\section{Intrduction}
Deep learning has been a great success in dealing with natural signals, e.g., images and voices, as well as artifact signals, e.g., nature language, while it is still in the early stage in handling sophisticated social and natural applications shaped by very diverse factors (e.g., stock market prediction), or resulted from complicated processes (e.g., pollution level prediction). One of distinctive features of the complicated applications is their applicable data sources are boundless. Consequently, their solutions need frequent revisions. Although neural networks can approximate arbitrary functions as close as possible  \cite{Hornik1991approximation}, 
the major reason for not existing such competent neural networks for those complicated applications is their problems are hardly fully understood and their applicable data sources cannot be identified all at once. By far the best practice is the developers pick a seemly neural network with available data to hope for the best. The apparent drawbacks, besides the performance, are the lack of flexibility in new data source emergence, better problem decomposition, and the opportunity of employing proven efforts from others.  On the other hand, some adopts a composition of several neural network models, based on function composition using domain knowledge.

An emerging trend of deep learning solution development is to employ  well crafted pre-trained neural networks (i.e., neural network models with instantiated weights), especially used as a component in a composited neural network model. Most popular pre-trained neural network models are well fine tuned with adequate training data, and made available to the public, either free or as a commercial product. During the training phase of composite neural network, the weights of pre-trained models are frozen to maintain its good quality and save the training time, while the weights of their outgoing edges are trainable. In some cases as in the transfer learning, the weights of pre-trained neural network are used as initial values in the training phase of composite neural network. 

It is intuitive that a composite neural network should perform better than any of its components. The ensemble learning \cite{freund1997decision,zhou2012ensemble} and the transfer learning \cite{galanti2016theoretical} have great success and are popular when pre-trained models are considered. 
However, in the transfer learning, how to overcome the negative transfer (a phenomenon of a pre-trained model has negative impact on the target task) is an important issue \cite{seah2013combating}. In the ensemble learning, it is well known that the adding more pre-trained models, it is not always true to have the better accuracy of the ensemble \cite{zhou2002ensembling}. Furthermore, \cite{opitz1999popular} pointed that the ensemble by boosting  having less accuracy than a single pre-trained model often happens for neural networks. In the unsupervised learning context, some experimental research  concludes that although layer-wise pre-training can be significantly helpful, on average it is slightly harmful \cite{goodfellow2016deep}. 
These empirical evidences suggest that in spite of the success of the ensemble learning and the transfer learning, the conditions that composite neural network can perform better is unclear, especially in the deep neural networks training process. 

The topology of a composite neural network can be represented as a rooted directed graph. For instance, an ensemble learning can be represented as 1-level graph, while a composite neural network with several pre-trained models that each is designed to solve a certain problem corresponds to a more complicated graph.
It is desired to discover a mathematical theory, in addition to employing domain knowledge, to construct a composite neural network with guaranteed overall performance.
In this work, we investigate the mathematical theory
to ensure the overall performance of a composite neural network is better than any a pre-trained component, regardless the way of composition, to allow deep learning application developer great freedom in constructing a high performance composite neural network.

\section{Preliminaries}
In this Section, we introduce some notations and definitions about composite neural network.
Parameters $N$,$K$, $d$, $d_j$, $d_{j_1}$, and $d_{j_2}$ are positive integers. Denote $\{1,...,K\}$ as $[K]$ and $[K]\cup\{0\}$ as $[K]^{+}$.
Let $\sigma: \mathbb{R}\to \mathbb{R}$ be a differentiable activation function, such as 
the Logistic function $\sigma(z)=1/(1+e^{-z})$ 
and the hyperbolic tangent $\sigma(z)=(e^{z}-e^{-z})/(e^{z}+e^{-z})$. 
For simplicity of notation, we sometimes abuse $\sigma$ as a vector value function.    
A typical one hidden layer neural network can be formally presented as
$w_{1,1}\sigma\left(\sum_{i=1}^{d} w_{0,i} \mathbf{x}_i+w_{0,0} \right)+w_{1,0}$.
We abbreviate it as $f_{\sigma,\mathbf{W}}(\mathbf{x})$, where $\mathbf{W}$ is the matrix defined by$w_{1,1},w_{1,0},...,w_{0,1},w_{0,0}$. 
Recursively applying this representation can obtain the neural network with more hidden layers. 
If there is no ambiguity on the activation function, then it can be skipped as $f_{\mathbf{W}}(\mathbf{x})$. 
Now assume a set of neural networks $\{f_{\mathbf{W}_j}(\mathbf{x}_j)\}_{j=1}^{K}$ is given, where $\mathbf{W}_j$ is the real number matrix defining the neural network $f_{\mathbf{W}_j}:\mathbb{R}^{d_{j_1}\times d_{j_2}}\to \mathbb{R}^{d_j}$, and $\mathbf{x}_j\in \mathbb{R}^{ d_{j_1}\times d_{j_2}}$ is the input matrix of the $j$th neural network. For different $f_{\mathbf{W}_j}$, the corresponding $d_{j}$, $d_{j_1}$ and $d_{j_2}$ can be different.
For each $j\in [K]$, let $D_j=\{(\mathbf{x}_j^{(i)},\mathbf{y}_j^{(i)})\in  \mathbb{R}^{(d_{j_1}\times d_{j_2})\times d_j} \}_{i=1}^{N}$ be a set of labeled data (for the $j$th neural network). For each $i\in [N]$, let $\mathbf{x}^{(i)}=(\mathbf{x}_1^{(i)},\dots,\mathbf{x}_K^{(i)})$, $\mathbf{y}^{(i)}=(\mathbf{y}_1^{(i)},\dots,\mathbf{y}_K^{(i)})$, and $D=\{ (\mathbf{x}^{(i)},\mathbf{y}^{(i)}) \}_{i=1}^{N}$. 

For a \textbf{pre-trained} model (component), we mean ${\mathbf{W}_j}$ is fixed after its training process, and then we denote $f_{\mathbf{W}_j}$ as $f_{j}$ for simplicity. On the other hand, a component $f_{\mathbf{W}_j}$ is \textbf{non-instantiated} means ${\mathbf{W}_j}$ is still free. A deep feedforward neural network is a hierarchical acyclic graph, i.e. a directed tree.  
In this viewpoint, a feedforward neural network can be presented as a series of function compositions. 
For given $\{f_{\mathbf{W}_j}(\mathbf{x}_j)\}_{j=1}^{K}$, we assume $\theta_j\in\mathbb{R}^{d_j}$, $j\in[K]$, which make the product $\theta_jf_{\mathbf{W}_j}(\mathbf{x}_j)$ is well-defined. Denote $f_0$ as the constant function $1$, then the liner combination with a bias is defined as as $\Theta(f_1,...,f_K)=\sum_{j\in[K]^{+}}{\theta_jf_j(\mathbf{x}_j)}$. 
Hence, an $L$ layers of neural network can be denoted as $\Theta_{(L)}\circ{\sigma}\circ \cdots \circ\Theta_{(0)} \left(\mathbf{x}\right)$. 
A composite neural network defined by components $f_{\mathbf{W}_j}(\mathbf{x}_j)$ can be designed as an directed tree. 
For instance, a composite neural network
 $\sigma_2\left(\theta_{1,0}+\theta_{1,1} f_{4}(\mathbf{x}_4)+ \theta_{1,2}\sigma_{1}(\theta_{0,0}+\theta_{0,1} f_{1}(\mathbf{x}_1)+\theta_{0,2} f_{\mathbf{W}_2}(\mathbf{x}_2)+\theta_{0,3} f_{3}(\mathbf{x}_3))\right)$ can be denoted as $\sigma_2 \circ\Theta_1\left(f_{4}, \sigma_1 \circ \Theta_0(f_{1},f_{\mathbf{W}_2},f_{3})\right)$, where $f_1$ and $f_3$ are pre-trained and $f_{\mathbf{W}_2}$ is non-instantiated. 
Note that in this work $D_j$ is the default training data of component $f_j$ of composite neural network, but $D_j$ can be different from the training data deciding the frozen weights in the pre-trained $f_j$.

Let $ \langle \vec{a} ,\vec{b} \rangle$ be the standard inner product of $\vec{a}$ and $\vec{b}$, and $||\cdot||$ be the corresponding norm. 
For a composite neural network, the training algorithm is the gradient descent back-propagation algorithm and the loss function is the $L_2$-norm of the difference vector. 
In particular, for a composite neural network $g_{\vec{\mathbf{\theta}}}$ the total loss on the data set $D$ is 
\begin{equation}
L_{\vec{\mathbf{\theta}}}\left(\mathbf{x};g_{\vec{\mathbf{\theta}}}\right)
= \langle \vec{g}_{\vec{\mathbf{\theta}}}\left(\mathbf{x}\right)-\vec{y}, \vec{g}_{\vec{\mathbf{\theta}}}\left(\mathbf{x}\right)-\vec{y} \rangle
= || \vec{g}_{\vec{\mathbf{\theta}}}\left(\mathbf{x}\right)-\vec{y} ||^2
\end{equation}
This in fact is $\sum_{i=1}^{N}{\left( g(\mathbf{x}^{(i)})-\mathbf{y}^{(i)}  \right)^2}$.
By the definition of $g_{\vec{\mathbf{\theta}}}(\cdot)$, this total loss in fact depends on the given data $\mathbf{x}$, the components defined by $\{{\Theta_j}\}_{j=1}^{K}$, the output activation $\sigma$, and the weight vector $\vec{w}$. 
Similarly, let ${L(f_j(\mathbf{x}_j))}$ 
be the loss function of a single component $f_i$. Our goal is to find a feasible $\vec{\mathbf{\theta}}$ s.t. $L_{\vec{\mathbf{\theta}}}\left(\mathbf{x};g\right)<\min_{j\in[K]}{L(f_j(\mathbf{x}_j))}$.

The total loss depends on the training data $\mathbf{x}$, the components defined by $\{h_j\}_{j=1}^{K}$, the output activation $\sigma$, and the weight vector $\mathbf{W}$. 
It is expected that a good composite network design has low $L_2$ loss, in particular lower than all its pre-trained components. Therefore, the goal is to find a feasible $\mathbf{\Theta}$ such that it meets the ``No-Worse'' property, i.e., $\mathcal{E}\left(g_{\mathbf{\Theta}}\right)<\min_{j\in[K]^{+}}{\mathcal{E}(f_j)}$.

\section{Theoretical Analysis}
The following assumptions are default conditions in the following proofs.
\begin{enumerate}
    \item[A1.] Linearly independent components assumption: \\ 
$\forall i\in[K]^{+}, \nexists \{\beta_j\}\subset \mathbb{R}, \mbox{ s.t. } \vec{f}_i=\sum_{j\in[K]\setminus \{i\}}{\beta_j\vec{f}_j}$. 
    \item[A2.] No perfect component assumption:\\ 
$\min_{j\in [K]}\left\{ \sum_{i\in[N]}{|f_j(\mathbf{x}_j^{(i)})-\mathbf{y}^{(i)} |} \right\}>0$. 
    \item[A3.] The activation function and its derivative are $C^1$-mappings (i.e., it is differentiable and its differential is continuous) and the derivative is non-zero at some points in the domain.
    \item[A4.] The training process is based on the stochastic gradient descent backpropagation (SGD-BP) algorithm~\cite{rumelhart1988learning}.
    \item[A5.] The number of components, $K$, is less than $2\sqrt{N}-1$, where $N$ is the size of the training data set.      
\end{enumerate}

\subsection{Single-Layer Composite Network} \label{Theory_training_rmse} 

The first theorem below states that if a single-layer composite network satisfies the above five assumptions, it meets the ``No-Worse'' property with high probability.

\begin{theorem}\label{theorem1}
Consider a single-layer composite network $g(\mathbf{x})=L_{(1)}(\sigma(L_{(0)}(f_1,...,f_K)))(\mathbf{x})$.
Then with probability of at least $1-\frac{K+1}{\sqrt{N}}$ there exists $\mathbf{\Theta}=\{\Theta_1,\Theta_0\}$ s.t. $\mathcal{E}_{\mathbf{\Theta}}\left(\mathbf{x};g\right)<\min_{j\in[K]}{\mathcal{E}(f_j(\mathbf{x}_j))}$.
\end{theorem}

We discuss two cases of the activation $\sigma$.
\begin{itemize}
    \item Case 1: $\sigma$ is a linear function.
    \item Case 2: $\sigma$ is not a linear function.
\end{itemize}

(\textbf{Case 1}) 
$\sigma$ is a linear activation such that a single-layer composite network such as $L_{(1)}(\sigma(L_{(0)}(f_1,...,f_K)))$ can be rewritten as a linear combination with bias, i.e., $g_{\mathbb{\theta}}(\mathbf{x})=\sum_{j\in[K]^{+}}{ \theta_jf_j(\mathbf{x}_j)}$ with a mean squared error of $\mathcal{E}_{\mathbf{\Theta}}\left(\mathbf{x};g\right)= \frac{1}{N}\sum_{i=1}^{N}(g_{\mathbf{\Theta}}(\mathbf{x}^{(i)})-\mathbf{y}^{(i)})^2$. 
Clearly, the composite network $g_{\mathbb{\theta}}$ should have a mean squared error equal to or better than any of its components $f_j$, as $g_{\mathbb{\theta}}$ can always act as its best component.
To obtain the minimizer $\mathbf{\Theta}^{*}$ 
 for the error $\mathcal{E}_{\mathbf{\Theta}}\left(\mathbf{x};g\right)$, we must compute the partial differential ${\partial \mathcal{E}_{\mathbf{\Theta}}}/{\partial{\theta}_j}$
for all ${j\in[K]^{+}}$. After some calculations~\cite{horn1990matrix}, we have Eq~(\ref{eq2}).
\begin{equation} \label{eq2}
\mathbf{\Theta}^{*} = \left[ \mathbf{\theta}_j\right]_{j\in[K]^{+}} 
= 
\left[ \langle \vec{f_i},\vec{f_j} \rangle \right]_{i,j\in[K]^{+}}^{-1}\times
\left[ \langle \vec{f_i},\vec{y} \rangle \right]_{i\in[K]^{+}}
\end{equation}
Since Assumption A1 holds, the inverse matrix $\left[ \langle \vec{f_i},\vec{f_j} \rangle \right]_{i,j\in[K]^{+}}^{-1}$ exists and can be written down concretely to obtain $\Theta^{*}$ as in Eq.~(\ref{eq2}). Lemma~\ref{lemma_1} summarizes the above arguments.

\begin{lemma}\label{lemma_1}
Set ${\Theta}^{*}$ as in Eq.~(\ref{eq2}); then
\begin{equation}\label{eqlemma1}
\mathcal{E}(g_{\mathbf{\Theta}^{*}})\leq \min_{j\in[K]^{+}}\{\mathcal{E}(f_j)\}.
\end{equation}
\end{lemma} 
\begin{proof} 
(of Lemma \ref{lemma_1}) \\
Recall that in the case of linear activate function, $g(\mathbf{x})= L(f_1,...f_K)=\sum_{j\in[K]^{+}}{ \theta_j}f_j(\mathbf{x}_j)$. Also recall that 
$
\mathcal{E}_{\mathbf{\Theta}}(\mathbf{x};g)
=\sum_{i=1}^{N}{( g(\mathbf{x}^{(i)})-y^{(i)} )^2}.
$ 
To prove the existence of the minimizer, it is sufficient to find the critical point for the deferential of Eq. (1). That is, to calculate the solution the set of equations: 
$$
\nabla_{ \mathbf{\Theta} } \mathcal{E}\left(\mathbf{x};g\right)
={\left[\begin{array}{c}
        \frac{\partial \mathcal{E}}{\partial\theta_0} \\ 
        \vdots \\
        \frac{\partial \mathcal{E}}{\partial\theta_K} \end{array}\right]} 
={\left[\begin{array}{c}
     0\\
     \vdots \\
     0  \end{array}\right]}, 
$$
where for each $s\in[K]^{+}$, and
$$\aligned
\frac{\partial \mathcal{E}}{\partial\theta_s}
=&2\sum_{i=1}^{N}{\left( g(\mathbf{x}^{(i)})-y^{(i)}  \right)\cdot f_s(\mathbf{x}^{(i)})} \\
=&2\sum_{i=1}^{N}{\left(\sum_{j\in[K]^{+}}{\theta_j}f_j(\mathbf{x}^{(i)}_j)-y^{(i)}  \right)\cdot f_s(\mathbf{x}^{(i)})} \\
=& 2\left(\sum_{j\in [K]^{+}}{\theta_j \langle \vec{f_s},\vec{f_j} \rangle} -\langle \vec{f_s},\vec{y} \rangle \right).
\endaligned$$
Hence, to solve  
$
\nabla_{ \mathbf{\Theta} } \mathcal{E}\left(\mathbf{x};g\right)
={\vec{0}}
$
is equivalent to solve $\theta_t$s in the equation
$$
\left[ \langle \vec{f_s},\vec{f_t} \rangle \right]_{(K+1)\times(K+1)}\times \left[ {\theta}_t\right]_{(K+1)\times 1} = \left[ \langle \vec{f_s},\vec{y} \rangle \right]_{(K+1)\times 1}
$$
where the indexes $s,t$ 
are in $[K]^{+}$.

Note that linear independence of $\{\vec{f}_j\}_{j\in[K]^{+}}$ makes
$\left[ \langle \vec{f}_s,\vec{f}_t\rangle \right]_{(K+1)\times(K+1)}$ a positive-definite Gram matrix \cite{horn1990matrix}, which means the inversion $\left[ \langle \vec{f}_s,\vec{f}_t\rangle \right]_{(K+1)\times(K+1)}^{-1}$ exists.
Then the minimizer $\mathbf{\Theta}^{*}$ is solved:
\footnotesize{
\begin{equation}\label{Formula}
\left[ {\theta}_t\right]_{(K+1)\times 1} 
= 
\left[ \langle \vec{f_s},\vec{f_t} \rangle \right]_{(K+1)\times (K+1)}^{-1}\times
\left[ \langle \vec{f_s},\vec{y} \rangle \right]_{(K+1)\times 1}
\end{equation}
}\normalsize
The above shows the {existence} of the critical points.  It is easy to check that the critical point can only be the minimizer of 
the squared error $\mathcal{E}_{\mathbf{\Theta}}\left(\mathbf{x};g\right)$. Furthermore, we immediately have  
$\mathcal{E}(g_{\mathbf{\Theta}^{*}})\leq \min_{j\in[K]^{+}}\{\mathcal{E}(f_j)\}$.
\end{proof}

From Eq. (\ref{Formula}) of the above proof, we can compute the minimizer for the case of the linear activation.
\begin{corollary}\label{FormulaSolution} 
The closed form of the minimizer is: 
\footnotesize{$$
\mathbf{\Theta}^{*}=\left[ \mathbf{\theta}_j\right]_{(K+1)\times 1} 
= 
\left[ \langle \vec{f_i},\vec{f_j} \rangle \right]_{(K+1)\times(K+1)}^{-1}\times
\left[ \langle \vec{f_j},\vec{y} \rangle \right]_{(K+1)\times 1}.
$$} 
\end{corollary}

There is a $\leq$ constraint on the loss function $\mathcal{E}(g_{\mathbf{\Theta}^{*}})$ in Eq.~(\ref{eqlemma1}) that is replaced by $<$ and a probability bound. If $\mathbf{\Theta}^{*}$ is not a unit vector, it is obvious that  $\mathcal{E}(g_{\mathbf{\Theta}^{*}})$ must be less than any $\mathcal{E}(f_j)$. Therefore, we proceed to estimate the probability of $\mathbf{\Theta}^{*}=\vec{e}_{j^{*}}$, where $j^{*}\in[K]^{+}$.
\begin{equation}\label{eq4}
\forall i\in[K]^{+},
\frac{\partial \mathcal{E}}{\partial{\theta}_i}\\
\big\vert_{
\mathbf{\Theta}=\vec{e}_{j^{*}}}
= 2
\langle\vec{f}_{j^{*}}-\vec{y},\vec{f}_i \rangle \\
\end{equation}
Eq.~(\ref{eq4}) shows the gradient of the error function with respect to $\theta_i$ conditioned on $\mathbf{\Theta}^{*}=\vec{e}_{j^{*}}$, which is the inner products of the difference between ${f}_{j^{*}}$ (the output of $g_{\mathbf{\Theta}^{*}}$) and the ground truth $\vec{y}$, and the output of each pre-trained component $\vec{f}_i$.  
%
When the minimizer $\mathbf{\Theta}^{*}=\vec{e}_{j^{*}}$, all the differentials $\frac{\partial \mathcal{E}}{\partial{\theta}_i}$ must equal zero, i.e., $\langle\vec{f_{j^{*}}}-\vec{y},\vec{f_{i}} \rangle=0$, or $\vec{f_{j^{*}}}-\vec{y}$ is perpendicular to $\vec{f_{i}}$. 
The following Lemma~\ref{lemma_JL} is an implication 
\footnote{Also refers to the lecture note of Andoni and Razenshteyn {https://ilyaraz.org/static/class/scribes/scribe5.pdf}}
from the proof of the Johnson-Lindenstrauss Lemma~\cite{johnson1984extensions}. 
\begin{lemma} \label{lemma_JL}
For a large enough $N$ and given $\vec{u}\in\mathbb{R}^{N}$, there is a constant $c>0$, s.t. for $\eta=cos^{-1}(c/\sqrt{N})$,
\begin{equation}\label{JL}
\Pr_{\vec{v}\in\mathbb{R}^{N}}\left\{ | \angle_{\vec{u},\vec{v}} -\frac{\pi}{2}|\leq \eta\right\} \geq 1-\frac{1}{\sqrt{N}}
\end{equation}
where $\angle_{\vec{u},\vec{v}}$ is the angle between $\vec{u}$ and $\vec{v}$.
\end{lemma}
The Johnson-Lindenstrauss Lemma says that a randomly sampled unit vector $\vec{v}$ is approximately perpendicular to a given vector $\vec{u}$ with high probability in a high dimensional space. The complement of Eq.~(\ref{JL}) is
\begin{equation}\label{JL_complement}
\Pr_{\vec{v}\in\mathbb{R}^{N}}\left\{ | \angle_{\vec{u},\vec{v}} -\frac{\pi}{2}|{>} \eta\right\} {<} \frac{1}{\sqrt{N}}
\end{equation}

Note that angles $\angle_{\vec{y} ,\vec{f}}$ , $\angle_{\vec{f}-\vec{y},\vec{f}}$, and $\angle_{\vec{f}-\vec{y},-\vec{y}}$ are the three inner angles of the triangle such that 
$\angle_{\vec{y},\vec{f}}+\angle_{\vec{f}-\vec{y},\vec{f}}+\angle_{\vec{f}-\vec{y},-\vec{y}}=\pi $. 
From Lemma~\ref{lemma_JL}, as $\angle_{\vec{y},\vec{f}}$ is likely a vertical angle (i.e., $\pi/2$),  $\angle_{\vec{f}-\vec{y},\vec{f}}$ must be less likely to be a vertical angle, which implies  
$\Pr\{ \langle\vec{f}-\vec{y}, \vec{f} \rangle=0  \}\leq\Pr\{|\angle_{\vec{f}-\vec{y},\vec{f}}-\pi/2|<\eta \}$; thus, $\leq \Pr\{| \angle_{\vec{y},\vec{f}} -{\pi}/{2}|>\eta \}.$
The following Lemma~\ref{lemma_3} immediately follows Lemma~\ref{lemma_JL} and Eq.~(\ref{JL_complement}).
\begin{lemma}\label{lemma_3} 
Following Lemma~\ref{lemma_JL}, then for given $\vec{y}\in\mathbb{R}^N $, 
$$
\Pr_{\vec{f}\in\mathbb{R}^{N}}
\left\{ \langle\vec{f}-\vec{y},\vec{f} \rangle=0  \right\}< \frac{1}{\sqrt{N}}.
$$
\end{lemma}
\begin{proof} (of Lemma \ref{lemma_3}) \\
Apply Lemma \ref{lemma_JL} on the given $\vec{y}$ and randomly selected $\vec{f}$, then we have
$$\Pr_{\vec{f}\in\mathbb{R}^{N}}\left\{ | \angle_{\vec{y},\vec{f}} -\frac{\pi}{2}|\leq \eta\right\} \geq 1- \frac{1}{\sqrt{N}}.$$ 
Also note that vectors $\vec{y}$, $\vec{f}$ and $\vec{f}-\vec{y}$ form a triangle with the three inner angles $\angle_{\vec{y} ,\vec{f}}$ , $\angle_{\vec{f}-\vec{y},\vec{f}}$ and $\angle_{\vec{f}-\vec{y},-\vec{y}}$, which means $\angle_{\vec{y},\vec{f}}+\angle_{\vec{f}-\vec{y},\vec{f}}+\angle_{\vec{f}-\vec{y},-\vec{y}}=\pi$. 
Hence, for large $N$,
$$\aligned
&\angle_{\vec{y},\vec{f}}=\frac{\pi}{2} \Rightarrow \angle_{\vec{f}-\vec{y},\vec{f}}\neq\frac{\pi}{2} \\
&\Rightarrow \Pr\left\{ \angle_{\vec{y},\vec{f}}=\frac{\pi}{2} \right\}\leq 
 \Pr\left\{ \angle_{\vec{f}-\vec{y},\vec{f}}\neq\frac{\pi}{2} \right\} \\
&\Rightarrow \Pr\left\{ \angle_{\vec{y},\vec{f}}\approx\frac{\pi}{2} \right\}\leq 
 \Pr\left\{ \angle_{\vec{f}-\vec{y},\vec{f}}\not\approx\frac{\pi}{2} \right\} \\
&\Rightarrow 1-\frac{1}{\sqrt{N}}\leq \Pr\left\{ \angle_{\vec{y},\vec{f}}\approx\frac{\pi}{2} \right\}\leq 
 \Pr\left\{ \angle_{\vec{f}-\vec{y},\vec{f}}\not\approx\frac{\pi}{2} \right\} \\
\endaligned $$ 
This means there exists small enough $\eta>0$ s.t. 
$$\aligned
&1-\frac{1}{\sqrt{N}}\leq \Pr\left\{ |\angle_{\vec{y},\vec{f}}-\frac{\pi}{2}|\leq \eta \right\}\leq 
 \Pr\left\{ 
 |\angle_{\vec{f}-\vec{y},\vec{f}}-\frac{\pi}{2}|\geq \eta 
 \right\} \\ 
&\Rightarrow \frac{1}{\sqrt{N}} >  \Pr\left\{ |\angle_{\vec{f}-\vec{y},\vec{f}}-\frac{\pi}{2}|<\eta \right\}
\endaligned $$

In short, as $\angle_{\vec{f}-\vec{y},\vec{f}}$ is likely $\pi/2$, $\angle_{\vec{y},\vec{f}}$  must be less likely a vertical angle.   
Hence, $1-\frac{1}{\sqrt{N}}\leq\Pr\{|\angle_{\vec{f}-\vec{y},\vec{f}}-\frac{\pi}{2}|\leq\eta \} \leq \Pr\{| \angle_{\vec{y},\vec{f}} -\frac{\pi}{2}|>\eta \}.$
This comppletes the proof.
\end{proof}
{ }

Lemma~\ref{lemma_3} shows that the probability of the output of one component is perpendicular to the difference between itself and the ground truth. For $K$ components and a bias, Lemma~\ref{lemma_4} gives a worst bound.

\begin{lemma}\label{lemma_4}
$\Pr\left\{\mathcal{E}(g_{\mathbf{\Theta}^{*}})= \min_{j\in[K]^{+}}\{\mathcal{E}(f_j)\} \right\}< \frac{K+1}{\sqrt{N}}$, i.e., $\Pr\left\{\exists\Theta^*:\mathcal{E}(g_{\mathbf{\Theta}^{*}})< \min_{j\in[K]^{+}}\{\mathcal{E}(f_j)\} \right\}\geq 1-\frac{K+1}{\sqrt{N}}$.
\end{lemma}
\begin{proof} (of Lemma \ref{lemma_4}) \\
Observe that as $j^*$ is fixed and known,
$$\aligned
&\Pr\left\{\nabla_{ \mathbf{\Theta} } \mathcal{E}|_{\Theta^*=\vec{e_{j^*}}}
={\vec{0}} \right\}\\
&=\Pr\left\{\langle\vec{f}_{j^*}-\vec{y},\vec{f}_0 \rangle=0 \wedge\cdots\wedge
\langle\vec{f}_{j^*}-\vec{y},\vec{f}_K \rangle =0   \right\} \\
&\leq \Pr\left\{ \langle\vec{f}_{j^*}-\vec{y},\vec{f}_{j^*} \rangle =0 \right\}\\
&<\frac{1}{\sqrt{N}}
\endaligned $$
The last inequality is from Lemma \ref{lemma_3}.
But in general $j^*$ is unknown, 
$$\aligned
&\Pr\left\{\exists\mathbf{\Theta}^*:\mathcal{E}(g_{\mathbf{\Theta}^{*}})= \min_{j\in[K]^{+}}\{\mathcal{E}(f_j)\} \right\} \\
&=\Pr\left\{\exists j \in [K]^+ s.t. \nabla_{ \mathbf{\Theta} } \mathcal{E}|_{\Theta^*=\vec{e_{j}}}
={\vec{0}} \right\}\\
&\leq \Pr\left\{ \vee_{j=0}^{K}\left\{ \langle\vec{f}_j-\vec{y},\vec{f}_j \rangle=0\right\} \right\}  \\
&= (K+1)\Pr\left\{ \langle\vec{f}-\vec{y},\vec{f} \rangle =0 \right\} \\
&<\frac{K+1}{\sqrt{N}}
\endaligned $$
Hence,
$$
\Pr\left\{\exists\mathbf{\Theta}^*\in\mathbb{R}^{K+1}s.t.\mathcal{E}(g_{\mathbf{\Theta}^{*}})< \min_{j\in[K]^{+}}\{\mathcal{E}(f_j)\} \right\}
>1-\frac{K+1}{\sqrt{N}}
$$
\end{proof}
{ }

(\textbf{Case 2}) $\sigma$ is not a linear function. The idea of the proof is to find an interval in the domain of $\sigma$ such that the output of $L_{(1)}(\sigma(\cdot))$ approximates a linear function as close as possible. 
This means there is a setting such that the non-linear activation function performs almost as well as the linear one; since the activation  $L_{(1)}(\sigma(\cdot))$ acts like a linear function, the lemmas of Case 1 are applicable. The conclusion of this case is stated as Lemma~\ref{lemma_case2}, while we introduce important properties in Lemmas~\ref{Inverse Function} and \ref{TaylorLagrange} for key steps in the proof. 

Since  $\sigma$ satisfies Assumption~A3,  the inverse function theorem 
of Lemma~\ref{Inverse Function} is applicable.

\begin{lemma} \label{Inverse Function} (Inverse function theorem~\cite{rudin1964principles})\label{IFThm}\\
Suppose $\mu$ is a $C^1$-mapping of an open set $E\subset \mathbb{R}^n$ to $\mathbb{R}^n$, $\mu'(z_0)$ in invertible for some $z_0\in E$, and $y_0=\mu(z_0)$.  (I.e., $\mu$ satisfies Assumption~A3.) Then \\
(1) there exist open sets $U$ and $V$ in $\mathbb{R}^n$ such that $z_0\in U$, $y_0\in V$, $\mu$ is one-to-one on $U$, and $\mu(U)=V$;\\
(2) if $\nu$ is the inverse of $\mu$, defined in $V$ by $\nu(\mu(x))=x$ for $x\in U$, then $\nu\in C^1(V)$.
\end{lemma}


We also need the following lemma as an important tool.

\begin{lemma} \label{TaylorLagrange}
(Taylor's theorem with Lagrange remainder~\cite{courant2012introduction})\\
If a function $\tau(y)$ has continuous derivatives up to the ($l+1$)-th order on a closed interval containing the two points $y_0$ and $y$, then
$$
\tau(y)=\tau(y_0)+\tau^{(1)}(y_0)(y-y_0)+\cdots+\frac{\tau^{(l)}(y_0)}{l!}(y-y_0)^l+R_l
$$
with the remainder $R_l$ given by the expression for some $c\in [0,1]$: 
$$
R_l=\frac{\tau^{(l+1)}(c(y-y_0))}{(l+1)!}(y-y_0)^{l+1}.
$$
\end{lemma}

Let $l=1$, $\tau(y)$ be obtained such that \footnotesize{
\begin{equation}\label{eq_8}
  \tau(y)=\tau(y_0)+\tau^{(1)}(y_0)(y-y_0)+\frac{\tau^{(2)}(c(y-y_0))} {2!}(y-y_0)^2.  
\end{equation}}\normalsize  
The second-degree term can be used to bound the approximation error.

Now we are ready to give more details to sketch the proof of Case 2. 
Denote $\mathbf{\Theta}_0^{*}$ as the minimizer of Case 1, i.e., the corresponding $g_{\mathbf{\Theta}_0^*}=L^*_{(0)}(f_1,...,f_K)$ satisfies $\mathcal{E}(g_{\mathbf{\Theta}_0^{*}})< \min_{j\in[K]^{+}}\{\mathcal{E}(f_j)\}=\mathcal{E}(f_{j^*})$ with high probability, and denote $\mathbf{\Theta}_{\epsilon}=\{\Theta_{1,\epsilon},\Theta_{0,\epsilon}\}$ corresponding to 
\begin{equation}\label{g theta}
g_{\mathbf{\Theta}_{\epsilon}}=L_{(1),\epsilon}(\sigma(L_{(0),\epsilon}(f_1,...,f_K))),
\end{equation}
called the scaled $\sigma$ function. 
Lemma~\ref{lemma_case2} below states a clear condition of a linear approximation of a non-linear activation function.

\begin{lemma}\label{lemma_case2}
For the given $g_{\mathbf{\Theta}^*_{0}}$, $\{\mathbf{x}^{(i)}\}_{i\in [N]}$, and any $0<\epsilon\leq 1$, there exists $\mathbf{\Theta}_{\epsilon}=\{\Theta_{1,\epsilon},\Theta_{0,\epsilon}\}$ such that
\begin{equation}\label{eq_Lemma7}
   \forall i\in[N], |g_{\mathbf{\Theta}_{\epsilon}}(\mathbf{x}^{(i)})-g_{\mathbf{\Theta}^*_{0}}(\mathbf{x}^{(i)})|<\epsilon. 
\end{equation}
Furthermore, for small enough $\epsilon$, 
\begin{equation}\label{eq_Lemma7_2}
 \Pr\left\{\mathcal{E}(g_{\mathbf{\Theta}_{\epsilon}})< \min_{j\in[K]^{+}}\{\mathcal{E}(f_j)\} \right\} \geq 1-\frac{K+1}{\sqrt{N}}.
\end{equation}
\end{lemma}
\begin{proof} (of Lemma~\ref{lemma_case2}) \\
\textbf{For Eq. (8):} 
We first give a procedure of obtaining $g_{\mathbf{\Theta}_{\epsilon}}(\mathbf{x}^{(i)})$, then verify these settings in the procedure fit the conclusion of the first part: $\forall i\in[N]$, 
$|g_{\mathbf{\Theta}_{\epsilon}}(\mathbf{x}^{(i)})-g_{\mathbf{\Theta}^*_{0}}(\mathbf{x}^{(i)})|<\epsilon$. \\
\textbf{Procedure for Eq. (8):}\\
For the given $\epsilon$ and $\sigma(\cdot)$, we first find the following items based on the conclusions of Case 1 and Lemmas: \\ 
$g_{\mathbf{\Theta}^*_0}(\cdot)$. (By case 1)\\ 
$z_0\in \mathbb{R}$ s.t. $\frac{d}{dz}\sigma(z)\neq 0$. (By A3)\\
$U$ contains $z_0$. (By Lemma \ref{Inverse Function}) \\
$V$ contains $y_0$. (By Lemma \ref{Inverse Function}) \\
$\tau: V\to U$ s.t. $\forall z\in U$, $\tau(\sigma(z))=z$. (By Lemma \ref{Inverse Function}) \\
(Denote $y_0=\sigma(z_0)$, so $\tau(y_0)=z_0$.)\\
Then compute:
$$ \aligned
M_{g} & = \max\left\{1, \max_{i\in[N]}\{2\cdot|g_{\mathbf{\Theta}^*_0}(\mathbf{x}^{(i)})| \} \right\} \\
M_{\sigma} & = \max\left\{1,\sup_{z\in U}\{2\cdot\left(\frac{\sigma(z)-\sigma(z_0)}{z-z_0}\right)^2\} \right\}\\
M_{\tau} & = \max\left\{1,\sup_{z\in U}\{ |\tau^{(2)}(\sigma(z)-\sigma(z_0))| \} \right\} \\
M_{\gamma} & = \lceil\log_{2}(M_{g}M_{\sigma}M_{\tau}\epsilon^{-1}) \rceil+1 \\
\gamma_0 & = \sup_{z\in U}\left\{r=|z-z_0|: (z_0-r,z_0+r)\subset U \right\} \\
\gamma & = \min\left\{ \gamma_0, 2^{-M_{\gamma}} \right\}   \\
M_0 & =  \gamma^{-1} M_{g} \\
M_1 & = M_{\sigma}M_{\tau} 
\endaligned $$
Define: \\
$$ \aligned
L_{(0),\epsilon}(\mathbf{x}) &= M_0^{-1}g_{\mathbf{\Theta}^*_0}(\mathbf{x})+z_0 \\
L_{(1),\epsilon}(y) &= M_0\cdot\tau^{(1)}(y_0)\cdot y+
M_0\cdot\left(z_0-\tau^{(1)}(y_0)\cdot y_0 \right)\\
\endaligned $$

\textbf{Verification}: \\ 
First observe that $L_{(0),\epsilon}(\mathbf{x})$ is a linear combination with a bias, i.e., an affine mapping, since $g_{\mathbf{\Theta}^*_0}(\mathbf{x})$ itself is an affine mapping. 
Similarly,  $L_{(1),\epsilon}(y)$ is an affine mapping of $y$.

Next, for all $i\in[N]$, $L_{\Theta_{0,\epsilon}}(\mathbf{x}^{(i)})=M_0^{-1}g_{\mathbf{\Theta}^*_0}(\mathbf{x}^{(i)})+z_0\in (-\gamma+z_0,z_0+\gamma)\subset U$ since $\gamma\leq \frac{\gamma_0}{2}$ and $(-\frac{\gamma_0}{2}+z_0,z_0+\frac{\gamma_0}{2})\subset U $.
Hence, by Lemma \ref{Inverse Function}, $$
\tau\left(\sigma\left(L_{\Theta_{0,\epsilon}}(\mathbf{x}^{(i)}) \right) \right)= L_{\Theta_{0,\epsilon}}(\mathbf{x}^{(i)}).
$$ 

Now let $z\in (-\gamma+z_0,z_0+\gamma)$ and $y=\sigma(z)$, then by Lemma \ref{TaylorLagrange} and Eq. (\ref{eq_8}), 
$$\aligned 
& |\tau\left(y\right)-\left(\tau(y_0)+\tau^{(1)}(y_0)(y-y_0)\right)|\\
 & = \frac{\tau^{(2)}(c(y-y_0))}{2!} (y-y_0)^2 \\
 &<  2\cdot\sup_{z\in U}\left\{ |\tau^{(2)}(\sigma(z)-\sigma(z_0))| \cdot\left(\frac{\sigma(z)-\sigma(z_0)}{z-z_0}\right)^2\right\}\cdot(z-z_0)^2 \\
 & \leq M_{\tau}M_{\sigma}\gamma^2 =M_1\gamma^2
\endaligned $$ 
Replace $y$ with $\sigma(z)$ and simplify the expression in the absolute value symbol, then we have $\tau\left(y\right)=\tau\left(\sigma(z)\right)=z$. Furthermore,  
$\tau(y_0)+\tau^{(1)}(y_0)(y-y_0) = \tau^{(1)}(y_0)\cdot y+ \left(\tau(y_0)-\tau^{(1)}(y_0)\cdot y_0 \right)$. 
Then replace $z$ with $L_{\Theta_{0,\epsilon}}(\mathbf{x}^{(i)})$, and $\tau(y_0)$ with $z_0$,   
$$\aligned 
& | M_0^{-1}g_{\mathbf{\Theta}^*_0}(\mathbf{x})+z_0 -
\left\{ z_0 +\tau^{(1)}(y_0)\left(\sigma\left( L_{\Theta_{0,\epsilon}}(\mathbf{x}^{(i)}) \right)-y_0 \right) \right\} |\\
& < M_1\gamma^2
\endaligned $$ \\
This means that
$$\aligned 
&|g_{\mathbf{\Theta}^*_0}(\mathbf{x})-L_{\Theta_{1,\epsilon}}\left(\sigma\left( L_{\Theta_{0,\epsilon}}(\mathbf{x}^{(i)}) \right)\right)|<  M_0M_1\gamma^2 \\
& \Rightarrow
|g_{\mathbf{\Theta}^*_0}(\mathbf{x})-g_{\mathbf{\Theta}_{\epsilon}}(\mathbf{x})|<  M_0M_1\gamma^2 \\
\endaligned $$ 

Recall that $\gamma\leq 2^{-M_{\gamma}}<\frac{\epsilon}{M_{g}M_{\sigma}M_{\tau}}$. Hence,
$$
M_0M_1\gamma^2=\gamma^{-1}M_{g}M_{\sigma}M_{\tau}\gamma^2
=M_{g}M_{\sigma}M_{\tau}\gamma
<\epsilon
$$
This achieve the goal of the first part of this Lemma.

\textbf{For Eq. (9):} 
For the second part, we claim the following settings satisfy $\mathcal{E}(g_{\mathbf{\Theta}_{\epsilon}})\leq \frac{2\mathcal{E}(g_{\mathbf{\Theta}^{*}_0})+\mathcal{E}(f_{j^*})}{3}
<\mathcal{E}(f_{j^*})$. \\
\textbf{Procedure for Eq. (9):}\\
Compute and then set these:
$$\aligned
&M_2=\max_{i\in[N]}\left\{|g_{\mathbf{\Theta}^{*}_0}(\mathbf{x}^{(i)})-y^{(i)}|\right\}\\
&\epsilon=\frac{\mathcal{E}(f_{j^*})-\mathcal{E}(g_{\mathbf{\Theta}^{*}_0})}{4N(2M_2+1)}
\endaligned$$
\textbf{Verification:} \\
Observe that 
$$\aligned
&\mathcal{E}(g_{\mathbf{\Theta}^{*}_0})< \mathcal{E}(f_{j^*})
\Rightarrow \\
&\max_{i\in[N]}\left\{(f_{j^*}(\mathbf{x}^{(i)})-y^{(i)})^2-(g_{\mathbf{\Theta}^{*}_0}(\mathbf{x}^{(i)})-y^{(i)})^2\right\}>0 \\
&\mathcal{E}(g_{\mathbf{\Theta}^{*}_0})+ \frac{\mathcal{E}(f_{j^*})-\mathcal{E}(g_{\mathbf{\Theta}^{*}_0})}{3} 
= \frac{2\mathcal{E}(g_{\mathbf{\Theta}^{*}_0})+\mathcal{E}(f_{j^*})}{3}
<\mathcal{E}(f_{j^*}) 
\endaligned$$
Besides,
$$N\cdot(2M_2+1)\cdot\epsilon<\frac{\mathcal{E}(f_{j^*})-\mathcal{E}(g_{\mathbf{\Theta}^{*}_0})}{3} $$
and
$$\aligned
&|g_{\mathbf{\Theta}_{\epsilon}}(\mathbf{x})-g_{\mathbf{\Theta}^*_{0}}(\mathbf{x})|<\epsilon \\
&\Rightarrow |(g_{\mathbf{\Theta}_{\epsilon}}(\mathbf{x})-y)-(g_{\mathbf{\Theta}^*_{0}}(\mathbf{x})-y)|<\epsilon\\
&\Rightarrow 0\leq |g_{\mathbf{\Theta}_{\epsilon}}(\mathbf{x})-y|<|g_{\mathbf{\Theta}^*_{0}}(\mathbf{x})-y|+\epsilon \\
&\Rightarrow (g_{\mathbf{\Theta}_{\epsilon}}(\mathbf{x})-y)^2 < (|g_{\mathbf{\Theta}^*_{0}}(\mathbf{x})-y|+\epsilon )^2
\endaligned$$
Hence, based on above observations we have 
$$\aligned
\mathcal{E}(g_{\mathbf{\Theta}_{\epsilon}})
=&\sum_{i\in[N]}{(g_{\mathbf{\Theta}_{\epsilon}}(\mathbf{x}^{(i)})-y^{(i)})^2} \\
<&\sum_{i\in[N]}\{|g_{\mathbf{\Theta}_0^{*}}(\mathbf{x}^{(i)})-y^{(i)}|+\epsilon\}^2 \\
=&\sum_{i\in[N]}(g_{\mathbf{\Theta}_0^{*}}(\mathbf{x}^{(i)})-y^{(i)})^2 \\& +\sum_{i\in[N]}\left\{2\epsilon\cdot|g_{\mathbf{\Theta}_0^{*}}(\mathbf{x}^{(i)})-y^{(i)}|+\epsilon^2  \right\}\\
= &\mathcal{E}(g_{\mathbf{\Theta}_0^{*}})+\epsilon\cdot\sum_{i\in[N]}\left(2|g_{\mathbf{\Theta}_0^{*}}(\mathbf{x}^{(i)})-y^{(i)}|+\epsilon\right)\\
\leq & \mathcal{E}(g_{\mathbf{\Theta}_0^{*}})+\epsilon\cdot N\cdot\left(2M_2+1\right) \\
< & \mathcal{E}(g_{\mathbf{\Theta}^{*}_0})+ \frac{\mathcal{E}(f_{j^*})-\mathcal{E}(g_{\mathbf{\Theta}^{*}_0})}{3} \\
= & \frac{\mathcal{E}(f_{j^*})+2\mathcal{E}(g_{\mathbf{\Theta}^{*}_0})}{3} \\
< & \mathcal{E}(f_{j^*})
\endaligned $$
which means that $\mathcal{E}(g_{\mathbf{\Theta}_{\epsilon}})< \min_{j\in[K]^{+}}\{\mathcal{E}(f_j)\}$. The proof is complete.
\end{proof}

The proofs of Cases 1 and 2 above complete the proof of Theorem~\ref{theorem1}.

\subsection{Complicated Composite Network}
In the previous section we investigated the performance of a single-layer composite network comprising several pre-trained components connected by an activation function. Now we consider expanding the composite network in terms of width and depth. Formally, for a given pre-trained component $f_K$ and a trained composite network $g_{K-1}$ of $K-1$ components $(f_1,...,f_{K-1})$, we study the following two questions in this section.
\begin{itemize}
    \item[Q1:] (Adding width) By adding a new pre-trained component $f_K$, we define $g_{K}=L_{(1)}(\sigma(L_{(0)}(f_1,...,f_{K-1},f_K ))$. Is there $\Theta$ such that $\mathcal{E}(g_{K-1}) > \mathcal{E}_{\Theta}(g_{K})?$ 
    \item[Q2:] (Adding depth) By adding a new pre-trained component $f_K$, let $g_{K}=L_{(K)}(\sigma(L_{(K-1)}(g_{K-1},f_K))$. Is there $\Theta$ such that $\mathcal{E}(g_{K-1}) > \mathcal{E}_{\Theta}(g_{K})?$
\end{itemize}
Lemma~\ref{AddOne} answers Q1, and  we require Proposition \ref{AddOne_Kis2} as the base of induction to prove it.
\begin{lemma}\label{AddOne}
Set $g_{K}=L_{(1)}(\sigma(L_{(0)}((f_1,...,f_{K-1},f_K )))$. With probability of at least $1-\frac{K+1}{\sqrt{N}}$,
there is $\mathbf{\Theta}$ s.t. $\mathcal{E}\left(g_{K-1}\right)> \mathcal{E}_{\mathbf{\Theta}}\left(g_{K}\right)$.
\end{lemma}
\begin{proposition}\label{AddOne_Kis2}
Consider the case of only two pre-trained models $f_0$ and $f_1$. There exists $(\alpha_0,\alpha_1)\in \mathbb{R}^2$ s.t.
\small{
$$
 \sum_{i\in[N]}{(f_1(\mathbf{x}^{(i)})-y^{(i)})^2}>\sum_{i\in[N]}{\left(\alpha_0 f_0(\mathbf{x}^{(i)})+\alpha_1 f_1(\mathbf{x}^{(i)})-y^{(i)}\right)^2}
$$}\normalsize
with a probability of at least $1-\frac{2}{\sqrt{N}} $.
\end{proposition}\label{proposition}
\begin{proof}(of Proposition \ref{AddOne_Kis2})\\
Let 
$$\aligned
 &D(\alpha_0,\alpha_1)\\
 &=\sum_{i\in[N]}{(f_1(\mathbf{x}^{(i)})-y^{(i)})^2-\left(\alpha_0 f_0(\mathbf{x}^{(i)})+\alpha_1 f_1(\mathbf{x}^{(i)})-y^{(i)}\right)^2}.
\endaligned $$
First observe that $D(0,1)=0$ and hence if $\nabla D(0,1)\neq (0,0)$ then it is easy to know that $\exists (\alpha_0^{*},\alpha_1^{*})$ s.t. $D(\alpha_0^{*},\alpha_1^{*})>0$.
$$
\nabla D(\alpha_0,\alpha_1)=
-2\cdot 
\begin{bmatrix}
\langle \alpha_0 \vec{f}_0 +\alpha_1 \vec{f}_1-\vec{y},\vec{f}_0 \rangle \\
\langle \alpha_0 \vec{f}_0 +\alpha_1 \vec{f}_1-\vec{y},\vec{f}_1 \rangle
\end{bmatrix} 
$$
Then, by considering $(\alpha_0,\alpha_1)=(0,1)$ we have
$$
\nabla D(0,1)=
-2\cdot 
\begin{bmatrix}
\langle  \vec{f}_1-\vec{y},\vec{f}_0 \rangle \\
\langle  \vec{f}_1-\vec{y},\vec{f}_1 \rangle
\end{bmatrix} 
$$
Apply Lemma \ref{lemma_3},  
$$\aligned
&\Pr\left\{\nabla D|_{\Theta^*=\vec{e_{j^*}}}
={\vec{0}} \right\}\\
&\leq \Pr\{\exists j\in[1]^{+} s.t.\langle \vec{f}_j-\vec{y},\vec{f}_j\rangle=0 \}\\ &<\frac{2}{\sqrt{N}}
\endaligned$$
That is, 
$$\aligned
& \Pr\{\exists (\alpha_0,\alpha_1) s.t. D(\alpha_0,\alpha_1)>0 \} \\
& \geq \Pr\{\nabla D(0,1)\neq \vec{0} \} \\
& >1-\frac{2}{\sqrt{N}}
\endaligned$$
\end{proof}

\begin{proof}(of Lemma \ref{AddOne})\\
We first prove this lemma of linear activation, and then similar to previous section apply Lemma \ref{lemma_case2} to address the non-linear activation. For the linear activation, it can be proved by induction.

\textbf{Base case:} It is done in Proposition \ref{AddOne_Kis2}. 

\textbf{Inductive step:} Suppose as $J=k-1$ the statement is true. That is, $g_{k-1}=L_{\Theta}(f_1,...,f_{k-1})$ and with probability at least $1-\frac{K}{\sqrt{N}}$
, there is $\mathbf{\Theta}$ s.t.  $\mathcal{E}\left(g_{K-2}\right)> \mathcal{E}_{\mathbf{\Theta}}\left(g_{K-1}\right)$.
As $J=k$, let $f_0$ and $f_1$ in Proposition \ref{AddOne_Kis2} be $g_{k-1}$ and $f_k$ respectively. Then we have  $\alpha_0g_{k-1}+\alpha_1 f_k$ as the composite network. 
Repeat the argument in previous Proposition, then we can conclude with probability at least $1-\frac{k+1}{\sqrt{N}}$ there is $(\alpha_0,\alpha_1)$ s.t.  $\mathcal{E}\left(g_{K-1}\right)> \mathcal{E}_{\mathbf{\Theta}}\left(\alpha_0g_{k-1}+\alpha_1 f_k\right)$. Note that $\alpha_0g_{k-1}+\alpha_1 f_k$ is a possible form of $g_{K}$. So the statement holds. The details are as follows:  
\scriptsize{$$\aligned
 &D(\alpha_0,\alpha_1)\\
 &=\sum_{i\in[N]}{(g_{k-1}(\mathbf{x}^{(i)})-y^{(i)})^2-\left(\alpha_0 g_{k-1}(\mathbf{x}^{(i)})+\alpha_1 f_{k}(\mathbf{x}^{(i)})-y^{(i)}\right)^2}.
\endaligned $$}\normalsize 
First observe that $D(1,0)=0$ and hence if $\nabla D(1,0)\neq \vec{0} $ then it is easy to know that $\exists (\alpha_0^{*},\alpha_1^{*})$ s.t. $D(\alpha_0^{*},\alpha_1^{*})>0$.
$$
\nabla D(\alpha_0,\alpha_1)=
-2\cdot 
\begin{bmatrix}
\langle \alpha_0 \vec{g}_{k-1} +\alpha_1 \vec{f}_k-\vec{y},\vec{g}_{k-1}\rangle \\
\langle \alpha_0 \vec{g}_{k-1} +\alpha_1 \vec{f}_k-\vec{y},\vec{f}_k \rangle 
\end{bmatrix} 
$$
Then, 
$$
\nabla D(1,0) =
-2\cdot 
\begin{bmatrix}
\langle  \vec{g}_{k-1}-\vec{y},\vec{g}_{k-1} \rangle \\
\langle  \vec{g}_{k-1}-\vec{y},\vec{f}_k \rangle)
\end{bmatrix} 
$$
Apply Lemma \ref{lemma_4} and by Induction hypothesis, we have 
$$\aligned
& \Pr\left\{\nabla D|_{\Theta^*=\vec{e_{j^*}}}
={\vec{0}} \right\}\\
&\leq \Pr\{\langle \vec{g}_{k-1}-\vec{y},\vec{g}_{k-1} \rangle=0\} +\Pr\{\langle \vec{f}_{k}-\vec{y},\vec{f}_{k} \rangle=0 \}\\ 
&<\frac{k}{\sqrt{N}}+\frac{1}{\sqrt{N}}=\frac{k+1}{\sqrt{N}}
\endaligned$$
Thus, 
$$\aligned
& \Pr\{\exists (\alpha_0,\alpha_1) s.t. D(\alpha_0,\alpha_1)>0 \} \\
& \geq \Pr\left\{\nabla D|_{\Theta^*=\vec{e_{j^*}}}
\neq {\vec{0}} \right\} \\
& >1-\frac{k+1}{\sqrt{N}}
\endaligned$$
This completes the inductive step.

For the non-linear activation, repeat the argument of Lemma 7 to obtain a proper $g_{\mathbf{\Theta}_{\epsilon}}$ corresponding to the given $\epsilon$ and the linear mapping $g_{\mathbf{\Theta}_{0}^{*}}$, and a small enough $\epsilon$ can yield a proper $\mathbf{\Theta}_{\epsilon}$ that fits the conclusion of $\mathcal{E}(g_{K-1})>\mathcal{E}_{\mathbf{\Theta}_{\epsilon}}(g_{K})$.
The probability of existence is inherently obtained as the same as in Lemma 7. 
\end{proof}

Proposition~\ref{AddOne_Kis2} can be proved by solving the inequality directly for the case of $K=2$, and then generalizing the result to larger $K$ by induction with the help of Lemma~\ref{lemma_3} to prove Lemma~\ref{AddOne}. 
Adding a new component $f_K$ to a composite network $g_{K-1}$ as in Q2, the depth of resulting $g_K$ increments by 1. If $\vec{g}_{K-1}$ and $\vec{f}_K$ satisfy A1 and A2, consider $\{g_{K-1},f_K\}$ as a new set of $\{f_1,f_2\}$ in the same layer. Consequently, we can apply the arguments in Case 2 of Theorem \ref{theorem1} to show 
Lemma~\ref{addDeep} in the following, which answers Q2 and says the resulting $g_K$ has a minimizer $\mathbf{{\Theta}^*}$ such that with high probability the loss decreases.
\begin{lemma}\label{addDeep}
Set $g_{K}=L_{(1)}(\sigma(L_{(0)}((g_{K-1},f_K))$.  If $\vec{g}_{K-1}$ and $\vec{f}_K$ satisfy A1 and A2, then with a probability of at least $1-\frac{2}{\sqrt{N}}$, there is $\mathbf{\Theta}$ s.t.  $\mathcal{E}\left(g_{K-1}\right)> \mathcal{E}_{\mathbf{\Theta}}\left(g_{K}\right)$.
\end{lemma}
\begin{proof}(of Lemma \ref{addDeep})\\
Observe that for the given set of pre-tained components $\{f_j\}_{j\in[K]}$ and by the definition of $g_{K-1}$, $f_K$ is not a component of $g_{K-1}$. Hence, if the activation functions used in the construction of $g_{K-1}$ are all linear, the assumption A1 implies that $\vec{g}_{K-1}$ is linear independent of  $\vec{f}_{K}$. Furthermore, if there is at least one non-linear activation function used in the construction of $g_{K-1}$, then as $N$ is large enough, Lemma \ref{lemma_JL} implies that $\vec{g}_{K-1}$ and $\vec{f}_{K}$ are not parallel with a very high probability. This means the assumption that $\vec{g}_{K-1}$ is linear independent of  $\vec{f}_{K}$ is reasonable. Furthermore, this implies that the events $E_1:\exists\mathbf{\Theta}s.t.\mathcal{E}_{\mathbf{\Theta}}(g_{K})<min\{\mathcal{E}(g_{k-1}),\mathcal{E}(f_{k})\}$, and $E_2:\mathcal{E}(g_{K-1})<\cdots< \min_{j\in[K]^{+}}\{\mathcal{E}(f_j)\}$, are independent. Hence, $\Pr\{E_1|E_2\}=\Pr\{E_1\}$.
\tiny
\normalsize
\end{proof}

The proof of Lemma~\ref{addDeep} is similar to the proof of Case 2 in the previous sub-section.
Lemmas~\ref{AddOne} and \ref{addDeep} imply a greedy strategy to build a complicated composite network.
Recursively applying both lemmas, we can build a complicated composite network as desired. Theorem~\ref{DeeperWider} gives a formal statement of the constructed complicated composite network with a probability bound. The proof of Theorem~\ref{DeeperWider} is based on mathematical induction on layers and the worst case probability is over-estimated by assuming each layer could have up to $K$ components.
\begin{theorem}\label{DeeperWider}
For an $H$-hidden layer composite network with $K$ pre-trained components, there exists $\mathbf{\Theta}^*$ s.t.  
\small $$
\mathcal{E}_{\mathbf{\Theta}^*}(g)<\min_{j\in[K]^+}\{\mathcal{E}(f_j)\}
$$\normalsize
with a probability of at least $\left(1-\frac{K+1}{\sqrt{N}} \right)^H$.
\end{theorem}
\begin{proof}(of Theorem \ref{DeeperWider})\\
For a set of given $K$ pre-trained components, 
$g_{k}:=L_{(k)}(\sigma(L_{(k-1)}(\cdots L_{(1)}(\sigma(L_{(0)}(f_1,\cdots,f_K)))\cdots)))$ is one of possible 
$H$-hidden layer composite network. Hence obviously, 
$$\aligned
&\Pr\left\{\exists\mathbf{\Theta}^{*}:\mathcal{E}(g_{\mathbf{\Theta}^{*}})< \min_{j\in[K]^{+}}\{\mathcal{E}(f_j)\} \right\}\\
&\geq \Pr\left\{\mathcal{E}(g_{H})<\mathcal{E}(g_{H-1})<\cdots<\mathcal{E}(g_{1})< \min_{j\in[K]^{+}}\{\mathcal{E}(f_j)\} \right\}\\
&\geq \Pr\left\{\mathcal{E}(g_{1})< \min_{j\in[K]^{+}}\{\mathcal{E}(f_j)\} \right\} \\ &\mbox{ }\times \Pr\left\{\mathcal{E}(g_{2})<\mathcal{E}(g_{1})\mid\mathcal{E}(g_{1})< \min_{j\in[K]^{+}}\{\mathcal{E}(f_j)\} \right\}\times \cdots \times \\ 
&\mbox{ } \Pr\left\{\mathcal{E}(g_{H})<\mathcal{E}(g_{H-1})\mid \mathcal{E}(g_{H-1})<\cdots< \min_{j\in[K]^{+}}\{\mathcal{E}(f_j)\} \right\}\\
&= \left(1-\frac{K+1}{\sqrt{N}}\right)^H
\endaligned 
$$
The last inequality is based on the fact that 
$$
\min_{k\in[H]}\left\{ P_k \right\} \geq 1-\frac{K+1}{\sqrt{N}},
$$
where \small{
$$\aligned
P_k & =\Pr\left\{\exists\mathbf{\Theta}:\mathcal{E}(g_{k})<\mathcal{E}(g_{k-1})\mid \mathcal{E}(g_{k-1})<\cdots< \min_{j\in[K]^{+}}\{\mathcal{E}(f_j)\} \right\} \\
  & =\Pr\left\{\exists\mathbf{\Theta}:\mathcal{E}(g_{k})<\mathcal{E}(g_{k-1})\right\} \mbox{ by Lemma 9.}
\endaligned$$
}\normalsize
This completes the proof.
\end{proof}

\subsection{Compositing non-instantiated Components}
Now we first consider some of components are pre-trained and some are non-instantiated, and then investigate the hierarchical combination of both kinds of components. In particularly, a sompilfied composite network can be re-written as 
$g(\mathbf{x})=w_0\cdot\sigma\left( \mathbf{\theta}_1 f_1+\mathbf{\theta}_2 f_{\mathbf{W}_2} \right)+b_0$, 
where $f_1$ is a pre-trained component and $f_{\mathbf{W}_2}$ is non-instantiated. Since $\Theta_2$ is not fixed, it can not be checked that LIC and NPC assumptions are satisfied. On the other hand, after initialization, $f_{\mathbf{W}_2}$ can be seen as a pre-trained component at any a snapshot during training phase.

\begin{theorem}\label{mainresult4}
In the end of an weight updating iteration, if the components $f_1$ and $f_{\mathbf{W}_2}$ satisfy LIC and NPC assumptions, then with high probability $\vec{\mathbf{w}}$ updated in the next iteration can improve the loss.
\end{theorem}
\begin{proof}
Recall the training algorithm is the backpropagation algorithm. Also note that according to Eq. (1), the order of updating is $\vec{\mathbf{\theta}}$ first and then $\Theta_2$. 
We denote in the end of iteration $i$ the value of  $\vec{\mathbf{\theta}}$ and $\Theta_2$ as $\vec{\mathbf{\theta}}^{(iter=i)}$ and $\Theta_2^{(iter=i)}$, respectively. With randomized initialization, $\Theta_2$ is assigned as  $\Theta_2^{(iter=0)}$ before the execution of the iteration $1$. Then in each iteration $i\geq 1$, $g(\mathbf{x})$ is a combination of fixed parameter components. Hence this can reduce to the all pre-trained cases, and can apply Theorem 1 and 2.
\end{proof}

\section{Empirical Studies}\label{experiments}
\text
This section is to numerically verify the performance of composite network for two distinctively different applications, image classification and PM2.5 prediction. 
For image classification, we examined two pre-trained components, the ResNet50 \cite{he2016deep} from Keras and the SIFT algorithm\cite{lowe1999object} from OpenCV, running on the benchmark of ImageNet competition\cite{russakovsky2015imagenet}. 
For PM2.5 prediction, we implemented several models running on the open data of local weather bureau and environment protection agency to predict the PM2.5 level in the future hours.


\subsection{ImageNet Classification}
We chose Resnet50 as the pre-trained baseline model and the SIFT model as an auxiliary model to form a composite neural network to validate the proposed theory. The experiments are conducted on the 1000-class single-label classification task of the ImageNet dataset, which has been a well received benchmark for image classification applications. A reason to choose the SIFT (Scale-Invariant Feature Transform) algorithm is that its function is very different from ResNet and it is interesting to see if the performance of ResNet50 can be improved as predicted from our theory.

We trained the SIFT model using the images of ImageNet, and directed the output to a CNN to extract useful features before merging with ResNet50 output. 
In the composite model, the softmax functions of both ResNet50 and SIFT model are removed that the outputs of length 1000 of both models are merged before the final softmax stage.  During the training process of composite network, the weights of ResNet50 and SIFT model are fixed, and only the connecting weights and bias are trained.

The ResNet50 was from He et al. that its Top-1 accuracy in our context was lower than reported in \cite{he2016deep} since we did not do any fine tuning and data preprocessing. In the Figure \ref{3models}, it shows the composite network has higher accuracy than ResNet50 during almost the complete testing run. Table \ref{Image+Sift} shows the same result that the composite network performs better too. The experiment results support the claims of this work that a composite network performs better than any of its components, and more components work better than less components.

\begin{figure}
\centering{
\includegraphics[width=0.4\textwidth]{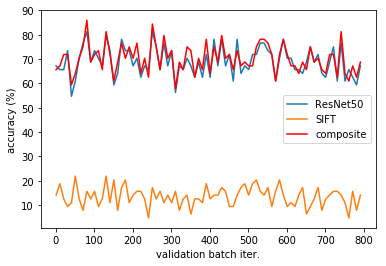}
}
\caption{Image Classification Validation Accuracy} 
\label{3models}
\end{figure}
\begin{table}
  \caption{Validation Error of Image Classification } 
  \label{Image+Sift}
  \centering \small{
  \begin{tabular}{lrrr}  
&&& \\
\toprule
      &  ResNet50 &  SIFT &  composite \\
\midrule 
number of parameter & 25636712  & 2884200  & 3000  \\
Validation Error ($\%$) & 70.1184  & 13.1252  & 71.5079  \\
\bottomrule
  \end{tabular} }
\end{table} \normalsize

\subsection{PM2.5 Prediction}
The PM2.5 prediction problem is to forecast the particle density of fine atmospheric matter with the diameter at most 2.5 $\mu m$ (PM2.5) in the future hours, mainly, for the next 12, 24, 48, 72 hours. The datasets used are open data provided by two sources including Environmental Protection Administration (EPA)\footnote{https://opendata.epa.gov.tw/Home} 
, and Center Weather Bureau (CWB)\footnote{http://opendata.cwb.gov.tw/index}. 
The EPA dataset contains 21 observed features, including the speed and direction of wind, temperature, relative humidity, PM2.5 and PM10 density, etc.,  from 18 monitoring stations, with one record per hour.
The CWB has seventy monitoring stations, one record per 6 hours, containing 26 features, such as temperature, dew point, precipitation, wind speed and direction, etc. 
We partitioned the observed area into a grid of 1140 km$^2$ with 1 km$\times$1 km blocks and aligned the both dataset into one-hour period. We called the two datasets as air quality and weather condition dataset. 

We selected ConvLSTM (Convolution LSTM) and FNN (fully connected neural network) as the components used in this experiment. The reason to select ConvLSTM is that  the dispersion of PM2.5 is both spatially and temporally dependent and ConvLSTM is considered capable of catching the dependency, and FNN is a fundamental neural network that acts as the auxiliary component in the experiment.

The prediction models were trained with the data of 2014 and 2015 years, then the 2016 data was used for testing. 
We considered two function compositions, the linear combination $\Theta$ and the Logistic function $\sigma_1$ (as Theorem 2), to combine the two components to examine the applicability of the proposed theorems.

\subsection{Pre-trained Component Design}
Here we introduce the design rationales of the five pre-trained components in this evaluation. 
As PM2.5 dispersion is highly spatially and temporally dependent, we designed four pre-trained components as base components to model this dependency. Among these, two were convolutional LSTM neural networks (ConvLSTMs~\cite{xingjian2015convolutional}) with the EPA data (denoted as $f_1$) and CWB data (denoted as $f_2$) as input; the other two were fully connected neural networks (FNNs) with the EPA data (denoted as $f_3$) and CWB data (denoted as $f_4$) as input. To model the temporal relationship conveniently using the neural network, the data was fed to the pre-trained components one sequence at a time. We used two pairs of components---$f_1$ and $f_2$, and $f_3$ and 
$f_4$---for the same functions to determine whether component redundancy improves performance.  The fifth pre-trained component (denoted as $f_5$) was to model the association between time and the PM2.5 value.

\begin{table}[ht]
  \centering \tiny{
  \caption{Various configurations of pre-trained components}
  \label{Table_DepthOptimal}
\begin{tabular}{lr|rr|rr|rr}  
\toprule
& {Forecast}   & \multicolumn{2}{c}{+24h}   & \multicolumn{2}{c}{+48h}  & \multicolumn{2}{c}{+72h} \\
               \cmidrule(r){3-4} \cmidrule(r){5-6}\cmidrule(r){7-8}  
Model & {Train.Params} &{Training} & {Testing}  & {Training} & {Testing} &  {Training} & {Testing}  \\
\midrule
$f_{1}$  & 917492 & 7.5873 & \textbf{10.5789}  &8.6541 & \textbf{11.3904}  & 8.8170 & \textbf{11.5279} \\ 
$f_{1,Wr}$  & 3632482 & 9.3054 & 11.9440   &9.1503 & 11.6550  & 8.1616 & 11.7556  \\ 
$f_{1,Dr}$  & 1278692 & 7.6342 & 10.9471   &8.6297 & 11.4844  &9.0803 & 11.5993 \\ 
\midrule
$f_{2}$  & 916908 & 8.5529 & \textbf{11.2074}   & 8.2890 & \textbf{11.7081}  & 9.2177 & \textbf{11.7756} \\ 
$f_{2,Wr}$  & 3631322 & 7.0685 & 11.4974   & 9.2233 & 12.0710  & 9.1766 & 11.9827 \\ 
$f_{2,Dr}$  & 790828 & 6.5404 & 11.7970   &8.4491 & 8.4491  & 9.1500 & 11.9162 \\ 
\midrule
$f_{3,(2)}$  & 1038054  &11.6064 & \textbf{10.8907} & 11.9008 & \textbf{11.6977}   &12.1729 & 11.9999  \\  
$f_{3,(3)}$  & 1068538  & 11.5648 &10.9179  & 11.9726 & 11.7017  & 12.0585 & \textbf{11.9414} \\ 
\midrule
$f_{4,(2)}$  & 582038 & 11.8238 & 11.3400 & 11.6948 & \textbf{11.6147}  & 11.9484 &11.8687 \\  
$f_{4,(3)}$  & 603318 & 11.8253 & \textbf{11.2748}  & 11.7112 & 11.6176  & 12.0199 &\textbf{11.7512} \\ 
\bottomrule
\end{tabular} }
\end{table} 

\begin{table}[ht]
\tiny{
\caption{Pre-trained components and testing RMSE}
  \label{Table_f1f2f3f4f5}
\begin{tabular}{llrrr}  
\toprule
Component & Data    & +24h   &+48h  & +72h \\
\midrule
$f_1$: ConvLSTM (2 CNN layers, 1 LSTM) & EPA& 10.5789   & 11.3904  & 11.5279 \\  
$f_2$: ConvLSTM (2 CNN layers, 1 LSTM) & CWB& 11.2074  & 11.7081  & 11.7756  \\  
$f_3$: FNN (2 hidden layers) & EPA & 10.6459  & 11.3291  & 11.6169 \\  
$f_4$: FNN (2 hidden layers) & CWB & 11.5112 &  11.6915  &  11.8017 \\  
$f_5$: LSTM & hr-week-month & 11.4738  &  11.5359 &  11.4540 \\  
\midrule
\multicolumn{5}{l}{EPA 9 features: CO, NO, NO2, NOx, O3, PM10, PM2.5, SO2, THC } \\
\multicolumn{5}{l}{CWB 5 features: AMB-TEMP, RH, rainfall, wind direction-speed (represented as a vector) }
\\
\bottomrule
\end{tabular} }
\end{table} 

There are five pre-trained components from $f_1$ to $f_5$ and one non-instantiated auxiliary component, denoted as $f_{W_6}$, for the composite network construction.  The model of $f_{W_6}$ is a convolutional neural network (CNN) with CWB weather data and forecasts as input to predict upcoming precipitation. The six components are connected by activation functions, either a linear function or a scaled logistic function  ($S(z)=2000/(1+e^{-z/500})-1000$). Note that any activation function that meets all six assumptions in Sec.~3 could be used; for simplicity, we used only the scaled logistic function. 
The prediction accuracy in RMSE of all five pre-trained components is listed in Table~\ref{Table_f1f2f3f4f5}. Note that in this study we did
not set out to design an optimized composite network for the best PM2.5 prediction. Rather, our main purpose was to implement and evaluate the proposed composite network theory. Nevertheless, the design of components and composite network follows the advice of domain experts and exhibits reasonably good performance in PM2.5 prediction.

\begin{table}[ht]
\hfill{} \tiny{
\caption{Summary of all methods (RMSE)}
  \label{Summary_all_methods}
\begin{tabular}{lr|rr|rr|rr}  
\toprule
&{}   & \multicolumn{2}{l}{+24h}   & \multicolumn{2}{l}{+48h}  & \multicolumn{2}{l}{+72h} \\
               \cmidrule(r){3-4} \cmidrule(r){5-6}\cmidrule(r){7-8}  
Method & {Trainable} &{Training} &{Testing}  & {Training} & {Testing} &  {Training} & {Testing}  \\
\midrule
SVM & - &  11.6440 & 10.9117 & 12.1246 & 11.5469 & 12.1670 &  11.6376     \\ 
Random forests & - & 3.3181  & 10.9386  & 3.4304 & 11.9037 &  3.4148 &  12.0917  \\ 
Ensemble & 1638 & 11.6955  & 11.0200  & 12.2609 & 11.3969 &  12.6605 &  11.6119  \\ 
\emph{SL}(Ensemble) & 1638  & 11.5855 & 10.9184 & 12.2080 & 11.2815 & 12.5690 &  11.5411    \\ 
DBCN$_{ \mathsf{Relu}}$ & 2664  & 12.4800 & 11.4540  & 13.3464 & 12.1947 &  14.0421 &  12.6546    \\ 
DBCN$_{ \mathsf{Sigm}}$ & 4032  & 11.7786 & 10.9803  & 13.6521 & 12.4418 &  13.4414 &  12.2825    \\
DBCN  & 2664  & 7.0560 & \textbf{10.2119}  & 8.0678 & \textbf{11.0469} &  8.2305 &  \textbf{11.4274}    \\ 
BBCN$_{\mathsf{Relu}}$   & 2664 & 13.3711 & 12.4575  & 14.6168 & 13.2662 &  15.8200 &  14.0754 \\ 
BBCN$_{\mathsf{Sigm}}$  &  4032 & 12.5376 & 11.4600  & 13.0951 & 12.2047 &  13.5416 & 12.0388 \\ 
BBCN  &  2664 & 7.1069 & 10.4712  & 7.9949 & 11.0935 &  8.4460 &  11.5100 \\ 
Exhaustive-a  & 2664  & 5.1646 & \textbf{9.2438}  & 5.0981 & \textbf{10.2402} &  6.7830 &  \textbf{10.4265}    \\ 
\midrule
\midrule
\multicolumn{8}{l}{(Include $f_{W_6}$)} \\ 
Ensemble  & 43684  & 11.5253 & 10.7338  & 12.4490 & 11.1874 &  12.5822 &  11.4804   \\
\emph{SL}(Ensemble)  & 43684  & 11.5117 & 10.8125  & 12.3939 & 11.1628 &  12.7025 &  11.3376   \\
DBCN$_{\mathsf{Relu}}$ & 44710  & 12.9434 & 11.8209  & 14.3413 & 12.8331 &  14.3562 & 12.7689    \\ 
DBCN$_{\mathsf{Sigm}}$  & 46420 & 11.9444 & 10.9167 & 12.1700 & \textbf{10.9474} &  13.2754 &  11.8630   \\ 
DBCN  &  44710 & 6.9705 & \textbf{10.1053}  & 7.8941 & 10.9531 &  8.2448 &  \textbf{11.2541}   \\
BBCN$_{\mathsf{Relu}}$   &  44710 & 11.4985 & 10.5742  & 12.0386 & 11.0392 &  12.7188 &  11.4047 \\ 
BBCN$_{\mathsf{Sigm}}$  & 46420 & 12.4675 & 11.3664  & 13.1786 & 11.9285 &  13.3815 & 11.8680 \\ 
BBCN &  44710  & 6.9828 & 10.1938  & 8.5736 & 11.0182 &  9.1848 &  11.4153  \\
Exhaustive-a  & 44710  & 5.5986 & \textbf{9.1971}  & 5.1292 & \textbf{10.2190} &  7.9572 &  \textbf{10.3588}    \\
Exhaustive-b &  44710 & 6.6991 & \textbf{9.2591}  & 5.6125 & \textbf{10.0632} &  5.7376 &  \textbf{10.2671}    \\
\bottomrule
\end{tabular} }
\hfill{}
\end{table}

We trained and tested both ConvLSTM and FNN using air quality dataset (Dataset A) and weather condition dataset (Dataset B) separately as the baselines (denoted as $f_1$, $f_2$, $f_3$ and $f_4$) and their training error and testing error in MSE are list in the first part of Table \ref{Errors of PM2.5}. Then we composited FNNs using Dataset A and Dataset B, each FNN can be pre-trained (denoted as x) or non-instantiated (denoted as o). In addition, we used both linear and Sigmoid activation functions. As a result, we had eight combinations, as list in the part two. We treated ConvLSTM in the same way and the outcomes were in the part 3. Finally, we composited using one FNN and one ConvLSTM that each was the best in their category, and the resulting composite network was a tree of depth 2. For instance, the candidate of ConvLSTM of part 4 for 12 hours prediction was the 4th row (i.e., $\Theta$($f_3^{\circ}$,$f_4^{\circ}$)) of part 3. Their training and testing errors in MSE were listed in the part 4. 

From the empirical study results, it shows mostly the proposed theorems are followed. While the composite networks with all pre-trained components may not perform better than others in their category, (which is not a surprise), what we expect to see is after adding a new component, the composite network has improvement over the previous one. For example, the $\sigma\circ\Theta(f_3^{\times},f_4^{\times})$ has strictly better accuracy than both $f_3$ and $f_4$ for all future predictions. Another example is the NEXT 48 hr, $\sigma\circ\Theta(C^{\times},F^{\times})$ also has strictly better accuracy than both $C=\sigma\circ\Theta(f_3^{\circ},f_4^{\circ})$ and $F=\sigma\circ\Theta(f_3^{\circ},f_4^{\circ})$.

\section{Related Work}
In this section, we discuss related work in the literature from the perspective of the composite network framework and PM2.5 prediction. For the framework, the composite network is related to the methods such as ensemble learning~\cite{zhou2012ensemble}, transfer learning~\cite{erhan2010does,kandaswamy2014improving,yao2010boosting} and model reuse~\cite{yang2017deep,wu2019heterogeneous}. We will also discuss some representative work on air quality prediction.

\textbf{Ensemble Learning.}
Typical ensemble learning methods include bagging, boosting, stacking, and linear combination/regression. 
Since bagging groups data by sampling and boosting tunes the probability of data~\cite{zhou2002ensembling}, these frameworks are not similar to composite neural networks. However, there are fine research results that are instructive for accuracy improvement~\cite{dvzeroski2004combining,gashler2008decision,zhou2002ensembling}. For example, it is known that in the ensemble framework, low diversity between members can be harmful to the accuracy of their ensemble~\cite{dvzeroski2004combining,gashler2008decision}.  
In this work, we consider the neural network composition, but not data enrichment.

Among the ensemble methods, stacking is closely related to our framework. The idea of stacked generalization~\cite{wolpert1992stacked}, in Wolpert's terminology, is to combine two levels of generalizers. The original data are taken by several level-0 generalizers, after which their outputs are concatenated as an input vector to the level-1 generalizer. According to the empirical study of Ting and Witten~\cite{ting1999issues}, the probability distribution of the outputs from level 0, instead of their values, is critical to accuracy. Their experimental results also imply that  multi-linear regression is the best level-1 generalizer, and a non-negative weight restriction is necessary for regression but not for classification.
However, our analysis shows that activation functions that satisfy Assumption A3 
have a high probability guarantee of reducing the L2 error. In addition, our empirical evaluations show that the scaled logistic activation usually performs well. 

The work of Breiman~\cite{breiman1996stacked} restricts non-negative combination weights to prevent poor generalization errors and concludes that it is not necessary to restrict the sum of weights to equal 1. In \cite{hashem1997optimal}, Hashem shows that linear dependence of components could be, but is not necessarily always, harmful to ensemble accuracy, whereas our work allows a mix of pre-defined and non-instantiated components as well as negative weights to provide flexibility in solution design. 

\textbf{Transfer Learning.} 
In the context of one task with a very small amount of training data with another similar task that has sufficient data, transfer learning can be useful~\cite{pan2009survey}. 
Typically the two data sets---the source and target domains---have different distributions.
A neural network such as an auto-encoder is trained with source-domain data and the corresponding hidden layer weights or output labels are used for the target task. 
Part of transplanted weights can be kept fixed during the consequent steps, whereas others are trainable for fine-tuning~\cite{erhan2010does}. This is in contrast to the composite neural network, in which the pre-trained weights are always fixed. For multi-source transfer, boosting-based algorithms are studied in \cite{yao2010boosting}. 
Kandaswamy et al.~\cite{kandaswamy2014improving} propose cascading several pre-trained layers to improve performance. Transfer learning can be considered a special case of the composite neural network if the source-domain neural network is fixed during target training.

\textbf{Ensemble (Bagging and Boosting).}
Since the Bagging needs to group data by sampling and the Boosting needs to tune the probability of data \cite{zhou2002ensembling}, these frameworks are different from composite neural network. However, there are fine research results revealing many properties for accuracy improvement \cite{dvzeroski2004combining,gashler2008decision,zhou2002ensembling}. For example, it is known that in the ensemble framework, low diversity between members can be harmful to the accuracy of their ensemble \cite{dvzeroski2004combining,gashler2008decision}.  
In this work, we consider neural network training, but not   data processing.

\textbf{Ensemble (Stacking).} Among the ensemble methods, the stacking is closely related to our framework. The idea of stacked generalization \cite{wolpert1992stacked}, in Wolpert's terminology, is to combine two levels of generalizers. The original data are taken by several level 0 generalizers, then their outputs are concatenated as an input vector to the level 1 generalizer. According to the empirical study of Ting and Witten \cite{ting1999issues}, the probability of the outputs of level 0, instead of their values, is critical to accuracy. Besides, multi-linear regression is the best level 1 generalizer, and non-negative weights restriction is necessary for regression problem while not for classification problem.
In \cite{breiman1996stacked}, it restricts non-negative combination weights to prevent from poor generalization error and concludes the restriction of the sum of weights equals to 1 is not necessary \cite{breiman1996stacked}. In \cite{hashem1997optimal}, Hashem showed that linear dependence of components could be, but not always, harmful to ensemble accuracy, while in our work, it allows a mix of pre-defined and undefined components as well as negative weights to provide flexibility in solution design. 

\textbf{Model Reuse.}
In recent years some proposed frameworks emphasize the reuse of fixed models~\cite{yang2017deep, feng2018autoencoder, wu2019heterogeneous, ijcai2019-472}. In this framework, pre-trained models are usually connected with the main (i.e., target) model, and then the dependency is gradually weakened by removing or reducing the connections during the training process. In this way, the knowledge of the fixed model is transferred to the main model; the key point is that model reuse is different from transfer learning as well as the composite neural network. 

Pre-trained models are widely applied in applications of natural language processing to improve the generation ability of the main model, such as in BERT~\cite{devlin2018bert} and ELMo~\cite{peters2018deep}.   
Multi-view learning~\cite{zhao2017multi} is another method to improve generalization performance.  In this approach, a specific task owns several sets of features corresponding to different views, just like an object observed from various perspectives, and separate models are trained accordingly. Then, the trained models for different views are combined using co-training, co-regularization, or transfer learning methods.

\section{Conclusion}
In this work, we investigated the composite neural network with pre-trained components problem and showed that the overall performance of a composite neural network is better than any of its components, and more components perform better than less components. In addition, the developed theory consider all differentiable activation functions. 

While the proposed theory ensures the overall performance improvement, it is still not clear how to decompose a complicated problem into components and how to construct them into a composite neural network in order to have an acceptable performance. Another problem  worth some thinking is when the performance improvement will diminish (by power law or exponentially decay) even adding more components. However, in the real world applications, the amount of data, data distribution and data quality will highly affect the performance.

\bibliographystyle{IEEEtran} 
\bibliography{main.bib}  

\newpage
\appendix
The follows are notations and definitions in this paper.
\begin{itemize}
    \item $\mathbf{x}$, an input vector.
    \item $\mathbf{W}$, a matrix of weights in a neural network. 
    \begin{itemize}
        \item $\mathbf{W}_1$, a matrix of weights in a single hidden-layer network.
    \end{itemize}
    \item $\sigma: \mathbb{R}\to \mathbb{R}$, an activation function.
    \item $f_{\sigma,\mathbf{W_1}}(\mathbf{x})\triangleq w_{1,1}\sigma\left(\sum_{i=1}^{d} w_{0,i} \mathbf{x}_i+w_{0,0} \right)+w_{1,0}$,\\ a single-layer neural network.\\
    If there is no ambiguity, it can be shortened to
    \begin{itemize}
        \item $f_{\mathbf{W}}(\mathbf{x})$,  a non-instantiated component. 
        \item $f(\mathbf{x})$, a pre-trained component.
    \end{itemize}
    \item $N$, the number of training data
    \item $K$, the number of components; each component can be either non-instantiated or pre-trained.
    \item $[K]^{+}\triangleq \{0,1,\cdots,K\}$
    \item $h_j$ is either $f_{\mathbf{W}_j}$ or $f_j$.
    \item $h_0\triangleq f_0\triangleq \mathbf{1}$, a constant function. 
    \item $\mathbf{x}_j\in \{{x}_j^{(1)},...,{x}_j^{(N)}\}$, an input of $h_j$.
    \item $L(\Theta; h_1,...,h_K)\triangleq \sum_{j=0}^{K}{ \theta_jh_j(\mathbf{x}_j)}$, a linear combination with bias. 
    \item {\small $L_{\Theta_{(h+1)}}\left(\sigma_{(h+1)}\left( \cdots \sigma_{(1)}\left(L_{\Theta_{(0)}} \left(h_1,...,h_K\right)\right)\right)\right)$}\normalsize, a network with $h$ hidden layers.
    \begin{itemize}
        \item Example: a composite neural network\\ 
 $\sigma_{(2)} ( \theta_{1,0}+ \theta_{1,1} f_{4}(\mathbf{x}_4)+  \theta_{1,2}\sigma_{(1)}( \theta_{0,0}+ \theta_{0,1} f_{1}(\mathbf{x}_1)+ $ $\theta_{0,2} f_{\mathbf{W}_2}(\mathbf{x}_2)+  \theta_{0,3} f_{3}(\mathbf{x}_3)) )$\\ 
 can be denoted as \\
 $\sigma_{(2)}\left( L_{(1)}\left(f_{4}, \sigma_{(1)} \left( L_{(0)}(f_{1},f_{\mathbf{W}_2},f_{3}) \right)\right)\right)$, with  $\mathbf\Theta$s removed for simplicity.
    \end{itemize}
    \item $\langle {\cdot} ,{\cdot} \rangle$, the standard inner product.
    \item $\vec{f}_j\triangleq (f_j(\mathbf{x}^{(1)}),\cdots,f_j(\mathbf{x}^{(N)}))$ 
    \item $\vec{y}\triangleq (y^{(1)},\cdots,y^{(N)})$.
    \item $g_{\mathbf{\Theta}}$, a composite network.
    \item $\mathcal{E}_{\mathbf{\Theta}}\left(\mathbf{x};g_{\mathbf{\Theta}}\right)\triangleq \frac {\langle g_{\mathbf{\Theta}}\left(\mathbf{x}\right)-\vec{y}, g_{\mathbf{\Theta}}\left(\mathbf{x}\right)-\vec{y} \rangle}{N}$, the $L_2$ loss function.\\
    If there is no ambiguity, it can be shortened to
    \begin{itemize}
        \item $\mathcal{E}_{\mathbf{\Theta}}\left(\mathbf{x};g_{\mathbf{\Theta}}\right)$, or $\mathcal{E}\left(g_{\mathbf{\Theta}}\right)$, for non-instantiated component.
        \item $\mathcal{E}(\mathbf{x}_j;f_j)$, or $\mathcal{E}(f_j)$, for pre-trained component.
    \end{itemize}
    \item $\vec{e}_j\in \mathbb{R}^{K+1}$, an unit vector in the standard basis for $j\in[K]^{+}$. 
    \begin{itemize}
       \item $\vec{e}_0=(1,0,0,\cdots,0)$    
       \item $\vec{e}_1=(0,1,0,\cdots,0)$
    \end{itemize}
\end{itemize}

\newpage
{ }
\newpage
{ }

\begin{table}[h]
  \caption{Training and Testing Errors of PM2.5 Prediction}
  \label{Errors of PM2.5}
  \centering \small{
  \begin{tabular}{lrrrrrrrr}  

\toprule
      Model   & \multicolumn{2}{l}{Next 12 hr}   & \multicolumn{2}{l}{Next 24 hr}  & \multicolumn{2}{l}{Next 48 hr}  & \multicolumn{2}{l}{Next 72 hr}  \\
                \cmidrule(r){2-3} \cmidrule(r){4-5}\cmidrule(r){6-7}\cmidrule(r){8-9}  
  (input) & TarinError & TestError  & TarinError & TestError & TarinError & TestError & TarinError & TestError\\
    \midrule 
    $\mathbf{f}_1$: FNN-A  & 100.1812 & 92.8528  & 134.7095 & 118.6065  & 141.6287 & 136.8358 & 148.1807 & 143.9980 \\  
    $\mathbf{f}_2$: FNN-B  & 134.1137 & 120.0019  & 139.8016 & 128.5960 &  136.7693  & 134.9001 & 142.7637 & 140.8650 \\  
    $\mathbf{f}_3$:  \scriptsize{ConvLSTM-A} & 54.2775 & 88.8156 & 57.5677 & 111.9122 & 74.8937 & 129.7418 & 77.7394 & 132.8923 \\ 
    $\mathbf{f}_4$: \scriptsize{ConvLSTM-B}  & 67.8625 & 118.4351 & 73.1519 & 125.6062  & 68.7069 &137.0789& 84.9656 & 138.6642\\ 
\midrule
 $\Theta$($f_1^{\times}$,$f_2^{\times}$)  &  99.7005  & $^{F:}$90.0214  &  130.7800  & 115.9283 &  139.9744 & $^{F:}$132.4764 & 144.6826 &  $^{F:}$137.8403 \\  
 $\Theta$($f_1^{\times}$,$f_2^{\circ}$)  &  95.6804  & 93.0173  &  120.3185  & 117.9781 &  134.3893 & 134.0270 & 139.6226 &  140.5209 \\  
 $\Theta$($f_1^{\circ}$,$f_2^{\times}$)  & 95.8110   & 93.1131  &  121.9737  & 117.7771  &  134.0676 & 135.2255 & 136.2009 & 144.0116 \\  
 $\Theta$($f_1^{\circ}$,$f_2^{\circ}$)  &  101.1584  & 90.2671 &  126.6807  & 114.5264  & 132.6726  & 132.8069 & 139.2339 &  139.3322 \\  
 $\sigma\circ\Theta$($f_1^{\times}$,$f_2^{\times}$)  &  102.7556  & 90.6280 &  133.1453  &117.7397 &  135.9256 &133.2544 & 145.1052 & 139.2513  \\  
 $\sigma\circ\Theta$($f_1^{\times}$,$f_2^{\circ}$)  &  98.1241  & 93.1098 &  127.4999  &118.8107 & 135.1553  &134.1469 & 137.7562 & 142.1778 \\  
 $\sigma\circ\Theta$($f_1^{\circ}$,$f_2^{\times}$)  &  94.9931  & 91.4667 &  124.5461  &117.7332& 131.2684  &135.1281 &140.1604  & 144.5220  \\  
 $\sigma\circ\Theta$($f_1^{\circ}$,$f_2^{\circ}$)  & 98.2596   & 91.3646 & 124.4182  &$^{F:}$114.2274 & 134.5078  & 132.8316 & 138.0456 & 139.8351 \\  
\midrule
 $\Theta$($f_3^{\times}$,$f_4^{\times}$)  & 27.1760  & 85.8922  & 49.1624 & 108.3157  & 37.2116 & 123.2186 &60.6415  &  131.0565  \\  
 $\Theta$($f_3^{\times}$,$f_4^{\circ}$)  & 27.1519 & 81.7688  & 42.7932 &  104.2375 & 33.2831 & $^{C:}$110.4213 & 74.3055 & $^{C:}$110.9952 \\  
 $\Theta$($f_3^{\circ}$,$f_4^{\times}$)  &  27.4436  & 78.5360  &  44.3214 &  107.0898 & 31.4910 & 129.1829 & 68.4413 &  139.5661  \\  
 $\Theta$($f_3^{\circ}$,$f_4^{\circ}$)  &  26.3063  &$^{C:}$70.8670 & 40.1879   & 100.8474  &  25.3312
 & 119.1634 & 76.1782 & 120.6814  \\  
$\sigma\circ\Theta$($f_3^{\times}$,$f_4^{\times}$)  & 28.3844 & 84.9029 & 43.2981 &109.3709  & 31.2413 &123.0041  & 63.7286 & 130.4122  \\  
 $\sigma\circ\Theta$($f_3^{\times}$,$f_4^{\circ}$)  & 27.5981 & 80.7848 & 44.4197 &98.1051  & 29.7649 & 111.6793 &69.3182 & 117.0719  \\  
 $\sigma\circ\Theta$($f_3^{\circ}$,$f_4^{\times}$)  & 26.4125 & 78.3990 & 42.3181 & 103.3361 & 30.4183 & 128.4138 &64.4193 &136.6043   \\  
 $\sigma\circ\Theta$($f_3^{\circ}$,$f_4^{\circ}$)  &  26.5131  & 75.5778 & 42.3912 &$^{C:}$94.6242 & 27.5812 & 112.8075 & 70.5132 &  117.6480 \\  
\midrule
 $\Theta$($C^{\times}$,$F^{\times}$)  &  26.6556  & 70.1159  & 34.6885   & 92.2737 & 29.6484  & 107.6833 & 52.0060 &  $^{H:}$110.1283 \\  
 $\Theta$($C^{\times}$,$F^{\circ}$) & 24.2349  & 67.4414 &  31.7202  &90.7795 & 30.1328  &$^{H:}$105.1804 &46.9994  & 111.2227  \\  
 $\Theta$($C^{\circ}$,$F^{\times}$)    &  22.7651  & 75.9468 & 24.2132   &96.3126  & 24.4488 &114.1803 & 49.0663 & 117.2747  \\  
 $\Theta$($C^{\circ}$,$F^{\circ}$)  &  21.9103  & 68.0660  & 20.9072   & 91.9323  & 23.2868  & 112.8605 & 30.6875 &  113.6968 \\   
 $\sigma\circ\Theta$($C^{\times}$,$F^{\times}$)  & 26.5950 & 69.1897 &  34.7400  &92.4715 &  29.9215 & 108.7482 & 52.3708 & 111.5474  \\   
 $\sigma\circ\Theta$($C^{\times}$,$F^{\circ}$)  &  24.0223  & $^{H:}$66.4733 &  28.8401  & $^{H:}$90.7257 & 28.5033  & 108.8896 & 46.2711 & 110.1613 \\   
 $\sigma\circ\Theta$($C^{\circ}$,$F^{\times}$)  &  22.4443   & 83.5953 &  22.5040  & 96.4027  & 28.4714   &112.0727 &  35.3947 & 114.5947 \\    
 $\sigma\circ\Theta$($C^{\circ}$,$F^{\circ}$)  & 38.4899 & 67.1819 & 17.6041  &92.2343  & 33.9710  & 105.7977  & 40.2934 & 110.3585  \\   
\midrule
\multicolumn{9}{l}{Let $\mathbf{f}_5$ be non-instantiated CNN with future rain fall input.}\\
 $\sigma\circ\Theta$($H^{\times}$,$f_5^{\circ}$)  & 24.1487 & 65.5776 & 30.6243 & 87.2777  & 55.9261 &102.2878& 50.5158& 108.8087\\    
 $\sigma\circ\Theta$($H^{\circ}$,$f_5^{\circ}$)  & 58.0852 &68.8111 & 22.8776  & 90.2324  & 39.0996  &112.1639
& 36.0659&109.8240\\    
\bottomrule
\multicolumn{9}{l}{Part 1: pre-trained components $f_i$, $i\in [4]$. Part 2: composite $f_1$ and $f_2$ by linear $\Theta(\cdot)$ or logistic $\sigma\circ\Theta(\cdot)$; similar for Part 3-5 }\\
\multicolumn{9}{l}{$^{\times}$: un-trainable component, i.e. pre-trained. $^{\circ}$: trainable component (original weights was deleted). }\\
\multicolumn{9}{l}{The best model of Part 2 (/Part 3) was assigned as composite model $F$ (/$C$) which will be used in Part 4.}\\
\multicolumn{9}{l}{The best model of Part 4 was assigned as composite model $H$ which will be used in Part 5. }\\
\end{tabular} }
\end{table} \normalsize

\end{document}